\documentclass[runningheads]{llncs}
\usepackage[T1]{fontenc}
\usepackage{graphicx}
\usepackage{amsmath}
\usepackage{amssymb}
\usepackage{amsfonts}
\usepackage{booktabs}
\usepackage{multirow}
\usepackage[ruled,vlined,linesnumbered]{algorithm2e}
\usepackage{hyperref}   
\newcommand{\sol}{\mathrm{sol}}
\newcommand{\rel}{\mathrm{rel}}
\newcommand{\var}{\mathrm{var}}
\newcommand{\vars}{\mathrm{vars}}
\newcommand{\ASK}{\textsf{Ask}}
\newcommand{\FindC}{\textsf{FindC}}
\newcommand{\FASTCA}{FastCA}
\newcommand{\TORACLE}{T-Oracle}
\newcommand{\QUACQ}{QuAcq}
\newcommand{\MQUACQ}{MQuAcq}
\newcommand{\GROWACQ}{GrowAcq}
\begin{document}
\title{Learning Symbolic Constraint Representations from Examples:
A Neuro-Symbolic Approach}
\titlerunning{Learning Symbolic Constraint Representations from Examples}
\author{Nassim Belmecheri\inst{1,3} \and
Arnaud Gotlieb\inst{1} \and
Nadjib Lazaar\inst{2} \and
Helge Spieker\inst{1}}
\authorrunning{N. Belmecheri et al.}
\institute{Simula Research Laboratory, Oslo, Norway \\
\email{\{nassim,arnaud,helge\}@simula.no} \and
LISN, CNRS, Paris-Saclay University, France \\
\email{lazaar@lisn.fr} \and
Microsoft, Norway}
\maketitle
\begin{abstract}
Learning user-defined concepts as constraint networks has been extensively
studied in the constraint acquisition (CA) literature. However, existing
approaches typically rely on intensive interactions with a human oracle, making
the learning process costly in terms of time and number of queries. In this
paper, we propose a neuro-symbolic framework for automatic CA that significantly
reduces user involvement by introducing neural \emph{Oracle Transformer} models
which learn to emulate user responses and to generalize conceptual knowledge.
Trained on previously available examples, the learned oracle interacts with a
dedicated CA engine, \FASTCA{}, which systematically refines the oracle's
responses into a sound, consistent, and interpretable constraint network. This
neuro-symbolic interaction enables the recovery of structured symbolic models
from data without prior domain knowledge. Our results demonstrate that this
neuro-symbolic interplay effectively aligns data-driven pattern recognition with
symbolic reasoning, offering a robust approach to automating model construction
in combinatorial domains.

\keywords{Constraint acquisition \and Neuro-symbolic learning \and Transformers
\and Constraint programming.}
\end{abstract}
\section{Introduction}
Constraint Programming (CP) is a powerful paradigm for modeling and solving
combinatorial problems in domains such as scheduling and planning~\cite{Rossi.Book.2006}.
At its core, CP relies on a declarative representation of knowledge: a constraint
network composed of variables, domains, and constraints that specify valid
configurations. Although the solving process itself is highly automated, the
knowledge modeling step remains a critical bottleneck~\cite{Freuder.PACLP.1999,Frisch.IJCAI.2005}. Defining
the right set of constraints requires domain expertise, is time-consuming, and
often involves iterative refinement. To address this issue, the field of
Constraint Acquisition (CA) has emerged as a bridge between raw data and
declarative knowledge~\cite{aij17}. CA methods aim to automatically learn
constraint networks from labeled examples (passive CA) or through interaction
with a user or external system acting as an oracle (active CA).

While active CA has demonstrated promising results~\cite{BessiereCDHKNQSTW23,TsourosS20},
it still suffers from practical limitations. Most notably, the oracle becomes a
bottleneck: humans cannot reliably answer thousands of queries or tolerate delays
of several seconds between questions. Various strategies have been proposed to
mitigate this, such as reducing the number of queries~\cite{TsourosBG23},
parallelizing interactions with multiple users~\cite{aaai21}, or even
automating the oracle in domains where executable programs can serve as ground
truth, e.g., in program verification, where a postcondition checker acts as the
oracle~\cite{MenguyBGL25}. Recent advances in deep learning, and in particular the success of large
language models and Transformer-based architectures, have shown remarkable
abilities to abstract, generalize, and even reason over structured
data~\cite{pan2025transformers}. In this context, the question arises:

\begin{quote}
\emph{Can Transformer models learn, from examples alone, the kinds of
abstractions, generalizations, and structural regularities that humans typically
extract when modeling a problem as a set of constraints?}
\end{quote}

To address the scalability limits of constraint acquisition with human oracles,
we propose \emph{neural oracles}---generative models that learn the underlying
structure of general constraint satisfaction problems from labeled examples.
Unlike existing methods that rely on a human oracle, our approach implements a
closed-loop neuro-symbolic system where a symbolic learner autonomously interacts
with a neural surrogate. This transition from human to neural feedback facilitates
a high-throughput, low-latency learning environment. Our results confirm that this
neuro-symbolic approach not only achieves high accuracy in reconstructing
constraint networks but also exhibits robust generalization across unseen problem
instances, unifying neural generalization with symbolic structure.

\section{Preliminaries}
A CA-learner is an algorithm that acquires constraints from data. In query-based
CA, it relies on a shared vocabulary that represents common knowledge between the
learner and the oracle, the source of information. A vocabulary is a pair
$(X, D)$, where $X$ is a finite set of variables, and $D \subset \mathbb{Z}$ is a
finite domain. The CA-learner starts with a predefined vocabulary and aims to
acquire a set of constraints $C$, forming the constraint network $(X, D, C)$. Each
constraint $c \in C$ is represented as $\langle \var(c), \rel(c) \rangle$, where
$\var(c) \subseteq X$ is the scope of the constraint and
$\rel(c) \subseteq D^{|\var(c)|}$ is the constraint relation. The set $N(c)$
denotes the violating tuples, ensuring $\rel(c) \cup N(c) = D^{|\var(c)|}$.

A full assignment $\langle v_1, \dots, v_n \rangle$ is an element of $D^n$, while a
partial assignment $\langle v_{j_1}, \dots, v_{j_k} \rangle$ belongs to $D^k$, with
$0 < k < n$. An assignment $e$ (full or partial) satisfies a constraint $c$ if its
projection onto $\var(c)$ belongs to $\rel(c)$; otherwise $e$ violates $c$. A full
assignment $e$ satisfies $C$ if it satisfies all constraints in $C$, making it a
solution of $C$. The set of all solutions of $C$ is denoted $\sol(C)$. If
$\sol(C) \subseteq \sol(C')$ we say $C$ entails $C'$; if $\sol(C) = \sol(C')$ then
$C$ and $C'$ are equivalent.

In addition to the vocabulary, the CA-learner has a constraint language that
defines the types of constraints allowed. A constraint language is a set
$\Gamma = \{r_1, \dots, r_t\}$, where each relation $r_i$ has fixed arity. A
constraint bias $B$ is a finite subset of the extension of $\Gamma$ over the
vocabulary $(X, D)$; it represents the candidate constraints considered for
inclusion in the learned network. For bounded arity $k$, $|B| \leq n^k t$; for
binary networks it is bounded by $n^2 t$~\cite{aij17}.

Given a vocabulary $(X, D)$, a general concept is a Boolean function $f$ over
$D^{|X|}$ that assigns a value in $\{0, 1\}$ to each assignment $e$. A
representation of a concept $f$ is a constraint network $C$ such that
$f^{-1}(1) = \sol(C)$. The target concept refers to the concept representing the
target problem. The target network is a network $T$ that satisfies $T \subseteq B$
and represents $f_T$.

\begin{definition}[Concept Decomposition]
A general concept $f$ can be decomposed into a set of local concepts $f_i$, each
defined on a subset of variables $X_i \subseteq X$. We write $f_i \sqsubset f$ to
indicate that $f_i$ is part of the decomposition of $f$:
\begin{equation}
f(X) = \bigoplus_{f_i \sqsubset f} f_i(X_i), \qquad \text{where } X_i \subseteq X.
\end{equation}
In parallel with the CP representation, each constraint $c_i \in B$ corresponds to
a local concept $f_i$, where $f_i^{-1}(1) = \rel(c_i)$.
\end{definition}

\begin{definition}[Concept Projection]
Given a general concept $f$ defined over a set of variables $X$, the projection of
$f$ onto a subset $Y \subseteq X$, denoted $f[Y]$, is a Boolean function over
$D^{|Y|}$ defined as
$f[Y](X) = \bigoplus_{f_i \sqsubset f} f_i(X_i)$, where $X_i \subseteq Y$.
\end{definition}

While humans naturally partition combinatorial problems into constituent local
concepts, translating such a decomposition into a formal constraint satisfaction
problem requires a mechanism to discern the underlying local structure. This
process is contingent on access to a source that can validate the underlying local
concepts. In learning-based settings, this role is traditionally played by an
oracle~\cite{Angluin.ML.1988,Valiant.ACM.1984}.

\begin{definition}[Oracle]
An oracle is an information source that has knowledge of the local concepts $f_i$
that compose a target concept $f_T$, and can provide answers to their validity,
without necessarily being able to express these concepts explicitly as constraints.
\end{definition}

\section{Transformer-Based Oracle}
Scaling human interaction is often impractical, and in many practical scenarios the
available data consists solely of labeled instances rather than explicit symbolic
constraints. To what extent can a Transformer-based model mimic a human's ability
to abstract, generalize, and represent the structure of a problem from examples
alone? To explore this, we introduce \TORACLE{}, a neural oracle model based on
Transformers trained to classify complete and partial assignments as satisfying or
violating a target concept. \TORACLE{} enables efficient interaction in the CA loop
by detecting local consistency in partial assignments, a capability that aligns
with partial queries in CA~\cite{BessiereCDHKNQSTW23}.

\subsection{Architecture of \TORACLE{}}
The \TORACLE{} model is built on a Transformer-based
architecture~\cite{vaswani2017attention}, designed to efficiently process combinatorial
problem instances and generalize constraint patterns from examples.

\begin{definition}[Training Set]
\begin{sloppypar}
The training set used to train \TORACLE{} is $(E^+, E^-)$, where $E^+$ and $E^-$
denote the positive and negative examples, respectively. An example can be either
a full assignment ($e \in f_T$ for a positive example, or $e \notin f_T$ for a
negative one), or a partial assignment $e$ on a subset $Y \subset X$ (with
$e \in f_T[Y]$ or $e \notin f_T[Y]$, respectively).
\end{sloppypar}
\end{definition}

\paragraph{Domain Encoder.}
Following the architectural principles of the Transformer~\cite{vaswani2017attention}, we
employ a dual-embedding scheme to represent both partial and complete assignments.
This approach explicitly disentangles the assigned value of a variable from its
structural identity (position) within the problem instance.

\emph{Value Embedding.} Each discrete domain value $v \in \{0, \dots, d_{\max}\}$,
where $d_{\max}$ is the maximum domain size and $0$ represents an unassigned
variable, is mapped to a continuous vector space using a learned embedding matrix
$\mathbf{W}_v \in \mathbb{R}^{(d_{\max}+1) \times k}$:
\begin{equation}
\mathbf{E}_v = \mathbf{W}_v[v].
\end{equation}

\emph{Positional Embedding.} Since the order and position of variables are critical
for defining constraints, we add a learnable positional embedding matrix
$\mathbf{W}_p \in \mathbb{R}^{n \times k}$, whose $i$-th row $\mathbf{P}_i$ learns a
representation for the $i$-th variable's position.

\emph{Combined Input Embedding.} For a given partial assignment
$x = (x_1, \dots, x_n)$, the final input representation for the $i$-th variable is
the sum of its value embedding and its positional embedding:
\begin{equation}
\mathbf{e}_i = \mathbf{E}_{x_i} + \mathbf{P}_i.
\end{equation}
The sequence $(\mathbf{e}_1, \dots, \mathbf{e}_n)$ constitutes the input tensor for
the subsequent Transformer layers, facilitating a global receptive field over the
assignment space.

\paragraph{Cross-Variable Attention.}
We extend the Transformer with a cross-variable attention mechanism to identify
relevant variable scopes. This mechanism is guided by a constraint matrix
$\mathbf{M} \in \mathbb{R}^{n \times n}$ that captures dependencies between variable
pairs: $\mathbf{M}_{ij}$ represents the learned global probability that a structural
dependency or constraint exists between variables $X_i$ and $X_j$. This matrix is
learned end-to-end through backpropagation.

\paragraph{Value Interaction Matrix.}
To model the dynamic interaction between the specific values assigned to the
variables, we project the input embeddings $\mathbf{e}_i$ into Query
($\mathbf{q}_i = \mathbf{W}_Q \mathbf{e}_i$) and Key
($\mathbf{k}_i = \mathbf{W}_K \mathbf{e}_i$) vectors. The value-level interaction is
\begin{equation}
\mathbf{V}_{ij} = \sigma\!\left(\frac{\mathbf{q}_i \cdot \mathbf{k}_j}{\sqrt{d_k}}\right),
\end{equation}
where $\sigma$ is the sigmoid function. A high value $\mathbf{V}_{ij}$ suggests that
the values assigned to $X_i$ and $X_j$ are a noteworthy pair (e.g., conflicting or
strongly correlated).

\paragraph{Dependency Score.}
From this, we compute an average dependency score $S$ for a given partial
assignment. Let $\mathbf{M}^{*}$ be the submatrix of $\mathbf{M}$ corresponding to
the active (non-zero) variables in the assignment. The score $S$ is the mean of all
entries in this submatrix:
\begin{equation}
S = \frac{1}{N_{\text{active}}} \sum_{i,j} \mathbf{M}^{*}_{ij},
\end{equation}
where $N_{\text{active}}$ is the number of assigned variables. This score
represents the average level of conflict and guides the final classification.

\paragraph{Classification Component.}
The model's final output is produced by a gating mechanism. A Transformer path
first computes an initial probability $p$. This probability is then modulated by the
dependency score $S$ using a learnable threshold $\tau$:
\begin{equation}
\text{output} = p \times \sigma(\tau - S).
\end{equation}
This allows the model to suppress the output if a high dependency score $S$ is
detected, guiding the Transformer's prediction.

\subsection{Variants of the \TORACLE{}}
We consider three variants of the \TORACLE{}, differing in the nature of the
training data used.

\begin{description}
\item[\textrm{TO}$_1$: Full Assignments Only.] The \TORACLE{} is trained
exclusively on complete assignments, i.e. $E \subseteq D^X$. This facilitates the
training of a high-fidelity neural oracle. However, in combinatorial domains,
constructing a balanced set of positive and negative examples is often infeasible,
which limits the practicality of this variant.
\item[\textrm{TO}$_2$: Partial Assignments of Varying Sizes.] The \TORACLE{} is
trained on a mix of full and partial assignments of varying sizes, allowing it to
generalize from incomplete information and capture constraints at different levels
of abstraction.
\item[\textrm{TO}$_3$: Constraint-Specific Partial Assignments.] This variant
assumes partial knowledge of the target constraints. For each constraint
$c \in T$, the training data include positive and negative tuples restricted to its
scope, i.e. examples $e, e' \in E$ such that $e \in \rel(c)$ and $e' \in N(c)$.
\end{description}

\section{\FASTCA{}: A Time-Efficient CA Learner}
Unlike CA learners based on membership queries (such as
\textsc{Conacq}~\cite{aij17}), which assess complete assignments,
approaches relying on partial queries guarantee learnability with a polynomial
number of queries~\cite{BessiereCDHKNQSTW23}. In this section, we present \FASTCA{}, a
constraint acquisition learner designed to optimize total acquisition time and
inter-query latency rather than minimizing the absolute number of queries. This
design is specifically tailored for environments featuring automated or neural
oracles, where the marginal cost of a query is negligible.

\paragraph{Description.}
\FASTCA{} (Algorithm~\ref{alg:fastca}) takes as input a bias $B$ and returns a
learned constraint network $L \subseteq B$. The algorithm iterates until $B$ is
empty, removing at least one constraint per iteration. At each step, a constraint
$c \in B$ is tested by querying a partial assignment $e$ on $\vars(c)$ such that
$e \in N(c)$. If $\ASK(e)$ is positive, $c$ is removed from $B$. Otherwise,
$\FindC$ is invoked to identify a valid constraint on $\vars(c)$, eliminating from
$B$ all constraints whose scope is included in $\vars(c)$. We use $\FindC$ as
defined in~\cite{BessiereCDHKNQSTW23}. In normalized CSPs (at most one constraint per
scope), $\FindC$ identifies the target constraint with at most one query. Thus,
each iteration requires at most two queries, while removing one or two constraints.

\begin{algorithm}[t]
\caption{\FASTCA{}}\label{alg:fastca}
\KwIn{Bias $B$}
\KwOut{Learned network $L$}
$L \gets \emptyset$\;
\While{$B \neq \emptyset$}{
  $c \gets \mathrm{pick}(B)$\;
  $e \gets \mathrm{pick}(N(c))$\;
  \eIf{$\ASK(e) = \mathrm{Yes}$}{
    $B \gets B \setminus \{c\}$\;
  }{
    $\FindC(e, \vars(c), L, B)$\;
  }
}
\Return $L$\;
\end{algorithm}

\paragraph{Theoretical Analysis.}
\FASTCA{} is sound, complete and terminates, with worst-case query complexity
$|B|$.

\begin{proposition}[Soundness]
Given a bias $B$ and a target network $T \subseteq B$ representing concept $f_T$,
the learned network $L$ returned by \FASTCA{} satisfies $\sol(T) \subseteq \sol(L)$.
\end{proposition}
\begin{proof}
Assume there exists $e^{*} \in \sol(T) \setminus \sol(L)$. Then at least one
constraint $c \in L$ rejects $e^{*}$, and $c$ was learned after classifying a tuple
violating $c$ as negative. Since $\FindC$ is sound~\cite{BessiereCDHKNQSTW23}, any
constraint added to $L$ cannot reject a tuple belonging to $\sol(T)$. This
contradiction implies no such $e^{*}$ exists.
\end{proof}

\begin{proposition}[Completeness]
Given a bias $B$ and a target network $T \subseteq B$ representing concept $f_T$,
the learned network $L$ returned by \FASTCA{} satisfies $\sol(L) \subseteq \sol(T)$.
\end{proposition}
\begin{proof}
Assume there exists $e^{*} \in \sol(L) \setminus \sol(T)$ when \FASTCA{}
terminates. Then some constraint $c \in B$ rejects $e^{*}$ but was not learned in
$L$. Either (i) $c$ was selected and removed because $e^{*} \in \rel(c)$ was
classified as positive, so $c \notin L$; or (ii) $c$ was removed by $\FindC$, which
does not remove constraints rejecting examples accepted by $L$. In both cases $c$
cannot reject any example in $\sol(L)$, proving $\sol(L) \subseteq \sol(T)$.
\end{proof}

\begin{proposition}[Termination]
Given a bias $B$ and a target network $T \subseteq B$, \FASTCA{} terminates.
\end{proposition}
\begin{proof}
At each iteration a constraint $c \in B$ is removed explicitly or implicitly by
$\FindC$. Since $B$ is finite and at least one constraint is removed at each
iteration, the algorithm terminates after at most $|B|$ iterations.
\end{proof}

\begin{theorem}[Correctness]
Given a bias $B$ and a target network $T \subseteq B$ representing $f_T$, \FASTCA{}
returns a network $L$ such that $\sol(L) = \sol(T)$.
\end{theorem}
\begin{proof}
Follows from the soundness, completeness and termination propositions.
\end{proof}

\begin{theorem}[Query Complexity]
Given a bias $B$ of bounded arity $k$ and a target network $T \subseteq B$,
\FASTCA{} asks at most $|B|$ queries. Since $|B| \leq n^k t$, the query complexity
is $O(n^k t)$; for binary constraint networks this reduces to $O(n^2 t)$.
\end{theorem}

\begin{corollary}[Time Complexity]
The time complexity of \FASTCA{} is linear in the size of the bias $B$.
\end{corollary}

Unlike traditional CA learners such as \QUACQ{}~\cite{BessiereCDHKNQSTW23}, \FASTCA{} is
inherently robust to noisy oracle feedback. By exhaustively evaluating all
constraints within the bias, it maintains acquisition continuity despite mislabeled
queries: a false positive may prune valid constraints (diminishing recall), while a
false negative may retain redundant constraints (affecting precision), yet the
acquisition process remains stable. In contrast, \QUACQ{}-like methods follow a
rigorous sequential protocol; a single misclassification can invalidate the search
assumptions, leading to algorithmic divergence or failure. This robustness makes
\FASTCA{} particularly suitable for use with imperfect oracles.

\section{Experiments}
We evaluated \emph{TRAC} through three main experiments: (i) analysis of oracle
performance, (ii) evaluation of the CA learner \FASTCA{}, and (iii) assessment of
their combination. The full implementation---data generation, oracle training,
prediction and acquisition---together with a script to reproduce the experiments
below is publicly available at
\url{https://github.com/NassimBelmecheri/TRAC}.

\subsection{Benchmark Problems}
The benchmarks cover various problem structures and abstraction levels, ranging
from classical puzzles and random constraint networks to real-world scheduling and
timetabling problems.
\begin{itemize}
\item \textbf{Sudoku} and \textbf{JSudoku}: $81$ variables on a $9\times 9$ grid,
domains $\{1,\dots,9\}$, with $810$ (resp. $811$) binary $\neq$ constraints. The
bias includes $19{,}440$ binary arithmetic constraints.
\item \textbf{LSquare}: a $10\times 10$ Latin square, $100$ variables of domain
size $10$ and $900$ binary $\neq$ constraints; the bias contains $29{,}700$
constraints.
\item \textbf{JobShop}: job-shop scheduling with $20$ jobs and $3$ machines,
$120$ variables and a time horizon of $15$; the bias includes $57{,}120$
constraints on $\Gamma = \{=, \neq, \geq, \leq, <, >, x_i + c = x_j\}$.
\item \textbf{Murder}: a logic puzzle with $20$ variables of domain size $5$ and
$53$ binary constraints (unique solution); the bias contains $1{,}140$ constraints.
\item \textbf{Zebra}: Lewis Carroll's puzzle with a unique solution, $25$ variables
of domain size $5$, five cliques of $\neq$ constraints and $14$ additional
arithmetic/distance constraints; the bias has $4{,}950$ unary and binary
constraints.
\item \textbf{Rand122} and \textbf{Rand495}: synthetic random networks with $50$
variables, domain size $10$, and respectively $122$ or $495$ binary constraints;
the bias contains $12{,}250$ constraints.
\item \textbf{ExamTT}: exam timetabling with $48$ variables, domains of size $90$
and $1{,}128$ constraints; the bias contains $7{,}896$ constraints on an extended
arithmetic language.
\item \textbf{NurseR}: nurse rostering with $105$ variables, domains
$\{1,\dots,18\}$ and $885$ $\neq$ constraints; the bias is built from
$\Gamma = \{=, \neq, \geq, \leq, <, >\}$ and contains $32{,}760$ constraints.
\end{itemize}

\subsection{Experimental Protocol}
Each oracle is trained on $80\%$ of the dataset and evaluated on the remaining
$20\%$, using a random split. Datasets for each oracle are constructed as follows.
\textbf{TO$_1$}: for each benchmark we generate $5{,}000$ solutions and $5{,}000$
non-solutions; solutions are obtained using the PyChoco solver with randomized
variable/value selection, and non-solutions by permuting solutions to introduce at
least one violation. \textbf{TO$_2$}: $100$ solutions and $100$ non-solutions,
supplemented with $1{,}000$ partial positive and $1{,}000$ partial negative
examples, with scopes ranging from size $2$ up to $n$. \textbf{TO$_3$}: for each
constraint $c$ in the target network, a proportion $\theta$ of positive tuples from
$\rel(c)$ is selected, together with an equal number of tuples from $N(c)$, with
$\theta$ ranging from $20\%$ to $80\%$. The per-constraint sampling is balanced
across constraint \emph{types}, so that rare relations (e.g.\ equalities,
arithmetic or unary constraints) are not dominated by frequent ones (e.g.\ the
$\neq$ cliques of an all-different), which otherwise biases the oracle towards a
single relation.

Our empirical evaluation addresses the following research questions.
\begin{description}
\item[RQ1:] To what extent does optimizing for acquisition throughput in \FASTCA{}
reduce total execution time compared to query-minimizing frameworks like
\QUACQ{}?
\item[RQ2:] How well do the three \TORACLE{} variants classify complete and partial
examples?
\item[RQ3:] How do different combinations of oracle and CA-learner components impact
overall acquisition performance?
\item[RQ4:] To what extent can the transformer-based oracle TO$_3$ reliably assess
constraint satisfaction and emulate a symbolic constraint checker?
\end{description}

\subsection{Results}

\paragraph{RQ1 (Comparative Performance of \FASTCA{}).}
Table~\ref{tab:rq1} compares \QUACQ{}~\cite{BessiereCDHKNQSTW23},
\MQUACQ{}~\cite{TsourosS20}, \GROWACQ{}~\cite{TsourosBG23} and our exhaustive
learner \FASTCA{} using the number of queries (\#q), average query time
($\bar{t}$), and total acquisition time ($T$). \GROWACQ{} provides the best overall
trade-off between query count and runtime. \FASTCA{} targets fully automated
settings and does not aim to minimize the number of queries; it therefore issues
significantly more queries, which is acceptable when using fast, machine-based
oracles. Crucially, \FASTCA{} achieves by far the lowest average query time across
all benchmarks, often close to zero, resulting in competitive total acquisition
times. When oracle latency is negligible, reducing per-query overhead is more
important than minimizing the number of queries.

\begin{table}[t]
\centering
\caption{CA-learner comparison: number of queries (\#q), average query time
$\bar{t}$ (s), and total acquisition time $T$ (s). \FASTCA{} issues more queries
but achieves the lowest per-query time.}\label{tab:rq1}
\footnotesize
\setlength{\tabcolsep}{4pt}
\begin{tabular}{@{}llrrr@{}}
\toprule
Benchmark & Learner & \#q & $\bar{t}$ & $T$ \\
\midrule
\multirow{4}{*}{Sudoku}  & \QUACQ{}   & 9K  & 0.38 & 3298 \\
                         & \MQUACQ{}  & 7K  & 0.16 & 1023 \\
                         & \GROWACQ{} & 5K  & 0.03 & 166  \\
                         & \FASTCA{}  & 18K & 0.02 & 395  \\
\midrule
\multirow{4}{*}{LSquare} & \QUACQ{}   & 11K & 0.31 & 3344 \\
                         & \MQUACQ{}  & 8K  & 0.15 & 1254 \\
                         & \GROWACQ{} & 6K  & 0.04 & 260  \\
                         & \FASTCA{}  & 28K & 0.03 & 893  \\
\midrule
\multirow{4}{*}{JobShop} & \QUACQ{}   & 784 & 0.29 & 230  \\
                         & \MQUACQ{}  & 615 & 0.26 & 158  \\
                         & \GROWACQ{} & 510 & 0.45 & 231  \\
                         & \FASTCA{}  & 14K & 0.01 & 208  \\
\midrule
\multirow{4}{*}{Murder}  & \QUACQ{}   & 487 & 0.76 & 370  \\
                         & \MQUACQ{}  & 358 & 0.13 & 48   \\
                         & \GROWACQ{} & 333 & 0.01 & 3    \\
                         & \FASTCA{}  & 1K  & 0.00 & 3    \\
\midrule
\multirow{4}{*}{Zebra}   & \QUACQ{}   & 747 & 0.84 & 624  \\
                         & \MQUACQ{}  & 681 & 0.07 & 51   \\
                         & \GROWACQ{} & 510 & 0.04 & 19   \\
                         & \FASTCA{}  & 2K  & 0.01 & 17   \\
\bottomrule
\end{tabular}
\hfill
\begin{tabular}{@{}llrrr@{}}
\toprule
Benchmark & Learner & \#q & $\bar{t}$ & $T$ \\
\midrule
\multirow{4}{*}{Rand122} & \QUACQ{}   & 1K  & 0.64 & 891  \\
                         & \MQUACQ{}  & 1K  & 0.10 & 117  \\
                         & \GROWACQ{} & 1K  & 0.02 & 21   \\
                         & \FASTCA{}  & 7K  & 0.01 & 51   \\
\midrule
\multirow{4}{*}{Rand495} & \QUACQ{}   & 7K  & 0.25 & 1675 \\
                         & \MQUACQ{}  & 6K  & 0.09 & 489  \\
                         & \GROWACQ{} & 5K  & 0.03 & 140  \\
                         & \FASTCA{}  & 29K & 0.03 & 831  \\
\midrule
\multirow{4}{*}{ExamTT}  & \QUACQ{}   & 1K  & 1.13 & 1237 \\
                         & \MQUACQ{}  & 590 & 0.26 & 154  \\
                         & \GROWACQ{} & 489 & 0.06 & 28   \\
                         & \FASTCA{}  & 776 & 0.01 & 5    \\
\midrule
\multirow{4}{*}{NurseR}  & \QUACQ{}   & 11K & 0.28 & 3231 \\
                         & \MQUACQ{}  & 9K  & 0.16 & 1421 \\
                         & \GROWACQ{} & 5K  & 0.07 & 380  \\
                         & \FASTCA{}  & 31K & 0.03 & 980  \\
\midrule
\multicolumn{5}{@{}l@{}}{\scriptsize $\bar{t}$, $T$ in seconds.}\\
\bottomrule
\end{tabular}
\end{table}

\paragraph{RQ2 (Classification Capabilities of \TORACLE{} Variants).}
Table~\ref{tab:rq2} reports the classification performance (Accuracy / Recall /
Precision) of the three oracle variants in a highly imbalanced setting where
non-solutions dominate and Murder and Zebra admit a single valid solution. TO$_1$,
trained on complete assignments only, achieves consistently excellent results, with
perfect Recall on all benchmarks. TO$_2$, trained on both partial and complete
assignments, performs worse on structurally complex and sparse-solution problems
such as JobShop and Zebra. TO$_3$, trained on partial assignments aligned with
individual constraint scopes, achieves strong and stable performance across
benchmarks, reaching perfect or near-perfect Precision on sparse-solution problems
while maintaining high Accuracy, and consistently outperforms TO$_2$.

\begin{table}[t]
\centering
\caption{Oracle performance comparison. Each cell reports Accuracy / Recall /
Precision (\%).}\label{tab:rq2}
\setlength{\tabcolsep}{6pt}
\begin{tabular}{lccc}
\toprule
Benchmark & TO$_1$ & TO$_2$ & TO$_3$ \\
\midrule
Sudoku   & 99/100/98   & 97/96/100  & 100/100/100 \\
JSudoku  & 100/100/100 & 97/96/98   & 100/100/100 \\
LSquare  & 100/100/100 & 98/96/100  & 100/100/100 \\
JobShop  & 100/100/100 & 69/59/73   & 82/79/84    \\
Murder   & 88/100/84   & 80/78/82   & 86/75/97    \\
Zebra    & 100/100/100 & 74/74/74   & 98/96/99    \\
Rand122  & 100/100/100 & 100/100/100& 99/98/100   \\
Rand495  & 99/100/99   & 97/96/98   & 98/97/100   \\
ExamTT   & 100/100/100 & 88/89/91   & 100/100/100 \\
NurseR   & 100/100/100 & 99/99/100  & 100/100/100 \\
\bottomrule
\end{tabular}
\end{table}

\paragraph{RQ3 (Evaluating Oracle--Learner Pairings).}
Table~\ref{tab:rq3} reports the performance of all CA-learner / oracle pairings
across benchmarks, as (Accuracy, Recall, Precision) triples. Unless stated
otherwise, TO$_3$ here denotes the variant trained with $\theta = 80\%$, the upper
end of the $20$--$80\%$ range described in the experimental protocol. Several patterns
emerge. (i) The best overall pairing, (\FASTCA{} / TO$_3$), consistently achieves
the highest performance, often reaching $100\%$ across all metrics---including on
sparse-solution problems such as Murder, Rand122 and Rand495. (ii) (\FASTCA{} /
TO$_2$) also performs strongly, typically in the $80$--$98\%$ accuracy range. In
contrast, (\FASTCA{} / TO$_1$) finds it hard to classify the queried examples
correctly: trained solely on complete assignments, TO$_1$ cannot reliably judge
the partial assignments generated during acquisition, and its precision collapses
to $0$ across benchmarks. This confirms that scope-aware supervision (TO$_3$),
and to a lesser extent partial-assignment supervision (TO$_2$), is essential for a
reliable neural oracle in the acquisition loop. (iii) Incremental learners
(\QUACQ{}, \MQUACQ{}, \GROWACQ{}) are more sensitive to oracle errors, since small
misclassifications can propagate as the learned network is updated step by step.

\begin{table}[!htbp]
\centering
\caption{Comparative performance of CA learners using different oracles, as
(Acc, Rec, Prec).}\label{tab:rq3}
\setlength{\tabcolsep}{3pt}
\scriptsize
\begin{tabular}{llccc}
\toprule
Benchmark & Learner & TO$_1$ & TO$_2$ & TO$_3$ \\
\midrule
\multirow{4}{*}{Sudoku}  & \QUACQ{}   & (28,100,28) & (30,98,28)  & (33,100,29)  \\
                         & \MQUACQ{}  & (37,100,37) & (40,98,38)  & (46,100,41)  \\
                         & \GROWACQ{} & (36,100,36) & (44,94,38)  & (43,100,39)  \\
                         & \FASTCA{}  & (96,100,0)  & (94,100,94) & (100,100,100)\\
\midrule
\multirow{4}{*}{JSudoku} & \QUACQ{}   & (28,100,28) & (55,82,37)  & (55,82,37)   \\
                         & \MQUACQ{}  & (36,100,36) & (44,95,39)  & (44,95,39)   \\
                         & \GROWACQ{} & (36,100,36) & (49,91,41)  & (49,91,41)   \\
                         & \FASTCA{}  & (96,100,0)  & (95,100,95) & (99,99,100)  \\
\midrule
\multirow{4}{*}{LSquare} & \QUACQ{}   & (29,100,29) & (32,97,30)  & (35,100,31)  \\
                         & \MQUACQ{}  & (36,100,36) & (39,97,37)  & (44,100,39)  \\
                         & \GROWACQ{} & (38,100,38) & (46,91,41)  & (46,91,41)   \\
                         & \FASTCA{}  & (99,100,0)  & (97,100,97) & (100,100,100)\\
\midrule
\multirow{4}{*}{JobShop} & \QUACQ{}   & (36,100,36) & (49,82,40)  & (43,75,36)   \\
                         & \MQUACQ{}  & (31,100,31) & (61,65,42)  & (36,76,30)   \\
                         & \GROWACQ{} & (86,100,86) & (79,90,86)  & (86,99,87)   \\
                         & \FASTCA{}  & (97,98,0)   & (83,83,100) & (93,94,99)   \\
\midrule
\multirow{4}{*}{Murder}  & \QUACQ{}   & (36,100,36) & (62,85,48)  & (44,97,39)   \\
                         & \MQUACQ{}  & (40,100,40) & (69,87,58)  & (52,87,58)   \\
                         & \GROWACQ{} & (50,100,50) & (66,80,62)  & (58,96,54)   \\
                         & \FASTCA{}  & (89,94,0)   & (94,95,99)  & (97,97,100)  \\
\midrule
\multirow{4}{*}{Zebra}   & \QUACQ{}   & (41,100,41) & (52,58,43)  & (46,98,43)   \\
                         & \MQUACQ{}  & (46,2,46)   & (59,78,57)  & (46,96,43)   \\
                         & \GROWACQ{} & (46,2,46)   & (59,37,82)  & (68,89,66)   \\
                         & \FASTCA{}  & (85,88,0)   & (81,88,90)  & (86,98,85)   \\
\midrule
\multirow{4}{*}{ExamTT}  & \QUACQ{}   & (24,100,24) & (34,92,26)  & (34,100,27)  \\
                         & \MQUACQ{}  & (46,100,46) & (52,100,49) & (52,100,49)  \\
                         & \GROWACQ{} & (36,100,36) & (51,92,42)  & (51,92,42)   \\
                         & \FASTCA{}  & (83,100,0)  & (86,100,86) & (98,100,98)  \\
\midrule
\multirow{4}{*}{Rand122} & \QUACQ{}   & (40,100,40) & (68,30,74)  & (68,30,74)   \\
                         & \MQUACQ{}  & (40,100,40) & (64,51,55)  & (64,51,55)   \\
                         & \GROWACQ{} & (46,100,46) & (63,43,64)  & (63,43,64)   \\
                         & \FASTCA{}  & (98,100,0)  & (98,100,98) & (100,100,100)\\
\midrule
\multirow{4}{*}{Rand495} & \QUACQ{}   & (36,100,36) & (47,94,40)  & (41,100,38)  \\
                         & \MQUACQ{}  & (35,100,35) & (40,100,37) & (41,100,37)  \\
                         & \GROWACQ{} & (37,100,37) & (39,100,37) & (39,100,37)  \\
                         & \FASTCA{}  & (98,100,0)  & (98,100,98) & (100,100,100)\\
\midrule
\multirow{4}{*}{NurseR}  & \QUACQ{}   & (32,100,32) & (34,95,32)  & (37,100,34)  \\
                         & \MQUACQ{}  & (39,100,39) & (42,98,40)  & (46,100,42)  \\
                         & \GROWACQ{} & (41,100,41) & (54,89,47)  & (54,89,47)   \\
                         & \FASTCA{}  & (97,100,0)  & (97,100,97) & (100,100,100)\\
\bottomrule
\end{tabular}
\end{table}

\paragraph{RQ4 (TO$_3$ as a Constraint Checker).}
This experiment evaluates the ability of TO$_3$ to emulate a symbolic constraint
checker. Table~\ref{tab:rq4} reports its classification performance on individual
constraints of increasing arity---binary, ternary, and quaternary. Training data is
constructed from tuples over integer domains of size $100$, from which $20\%$ of
positives and $20\%$ of negatives are sampled and balanced. TO$_3$ achieves over
$95\%$ accuracy on binary constraints and maintains high performance on ternary and
quaternary constraints (often $>98\%$), confirming that it generalizes well across
constraint types and arities, making it a reliable, lightweight component for
learning-based CA systems.

\begin{table}[!htbp]
\centering
\caption{Evaluating TO$_3$ as a constraint checker.}\label{tab:rq4}
\setlength{\tabcolsep}{6pt}
\begin{tabular}{llrrr}
\toprule
Arity & Constraint & Accuracy & Recall & Precision \\
\midrule
\multirow{2}{*}{Binary}
  & $X > Y$          & 95.9 & 95.7 & 96.1 \\
  & $X \neq Y$       & 98.0 & 100  & 96.0 \\
\midrule
\multirow{4}{*}{Ternary}
  & $X + Y > Z$          & 99.1 & 98.6 & 99.5 \\
  & $X + Y \neq Z$       & 99.1 & 99.9 & 98.5 \\
  & $|X - Y| > |X - Z|$  & 99.1 & 98.7 & 99.6 \\
  & $|X + Y| \neq |X - Z|$ & 99.1 & 99.9 & 98.9 \\
\midrule
\multirow{4}{*}{Quaternary}
  & $X + Y > Z + T$        & 98.1 & 96.2 & 100  \\
  & $X + Y \neq Z + T$     & 98.8 & 100  & 97.6 \\
  & $|X - Y| > |Z - T|$    & 97.7 & 95.5 & 100  \\
  & $|X + Y| \neq |Z - T|$ & 99.0 & 100  & 98.1 \\
\bottomrule
\end{tabular}
\end{table}

\section{Conclusion}
In this work, we presented a neuro-symbolic framework for automated constraint
acquisition, integrating for the first time a Transformer-based oracle
(\TORACLE{}) with the \FASTCA{} symbolic learner. Our findings demonstrate that an
oracle trained on localized, constraint-level data (TO$_3$) effectively captures
the latent semantics of combinatorial problems. By pairing this neural guidance
with \FASTCA{}'s exhaustive search---which is resilient to the incorrectly
classified examples that typically destabilize sequential learners---we establish a
robust automated constraint acquisition framework. While this hybrid approach
successfully bridges the gap between neural generalization and symbolic reasoning,
future work remains to assess its scalability to extremely large-scale domains where
exhaustive search may become computationally expensive. Nevertheless, this framework
represents a significant step toward interpretable, data-driven model construction
without the need for a human oracle.

%
%
%
\bibliographystyle{splncs04}
\bibliography{aaai2026}

\end{document}